%% file: main.tex
\documentclass{document2col} 

\usepackage[american]{babel}

\usepackage{natbib} 

\usepackage{mathtools} 
\usepackage{booktabs} 
\usepackage{tikz} 

\usepackage{amssymb}
\usepackage{amsthm}
\usepackage{amsmath}
\usepackage{enumitem}
\usepackage{xspace}
\usepackage{cleveref}

\newtheorem{theorem}{Theorem}
\newtheorem{lemma}{Lemma}
\newtheorem{conjecture}{Conjecture}

\usepackage{tikz}
\usetikzlibrary{positioning}
\usetikzlibrary{positioning,fit}

\theoremstyle{definition} 
\newtheorem{definition}{Definition}[section]

\newcommand{\skel}[1]{\text{skeleton of } #1}

\usepackage{xparse}
\NewDocumentCommand{\ind}{o}{\mathrel{\mathpalette{\independenT}{\perp}}\IfValueT{#1}{_{#1}}}
\newcommand{\independenT}[2]{\rlap{$#1#2$}\mkern2mu{#1#2}}

\newcommand{\mnsl}[0]{\textsc{Markov Network Learning with Errors}}
\newcommand{\bnsl}[0]{\textsc{Bayesian Network Learning with Errors}}
\newcommand{\cbn}[0]{\textsc{Closest Bayesian Network within a Graph Class}}
\newcommand{\mabbr}[0]{$k$-\textsc{MNSL}\xspace}
\newcommand{\babbr}[0]{$k$-\textsc{BNSL}\xspace}

\title{Learning Bayesian and Markov Networks with an Unreliable Oracle}

\author[1]{Juha Harviainen}
\author[2]{Pekka Parviainen}
\author[3]{Vidya Sagar Sharma}
\affil[1]{University of Helsinki, Finland}
\affil[2]{University of Bergen, Norway}
\affil[3]{Indian Institute of Technology Madras, India}
\date{}

\begin{document}
\maketitle

\begin{abstract}
    We study constraint-based structure learning of Markov networks and Bayesian networks in the presence of an unreliable conditional independence oracle that makes at most a bounded number of errors. For Markov networks, we observe that a low maximum number of vertex-wise disjoint paths implies that the structure is uniquely identifiable even if the number of errors is (moderately) exponential in the number of vertices. For Bayesian networks, however, we prove that one cannot tolerate any errors to always identify the structure even when many commonly used graph parameters like treewidth are bounded. Finally, we give algorithms for structure learning when the structure is uniquely identifiable. 
\end{abstract}

\section{Introduction}

Probabilistic graphical models are used to represent joint distributions of several random variables. Typically, they consist of two parts that are a graphical structure expressing how the distribution factorizes and parameters specifying the distribution.
There are several different types of probabilistic graphical models, among which this work concentrates on two commonly used models: Markov networks and Bayesian networks. The graph structure in a Markov network is an undirected graph and in a Bayesian network it is a directed acyclic graph (DAG).

Both Markov and Bayesian networks are often learned from data, with the present work focusing on the constraint-based approach to learning. There, one conducts conditional independence tests and constructs a graph that expresses conditional independencies that match the test results. After one has learned the structure, the (distribution of) parameters can be estimated. We study only the first task, which is called structure learning.

Typically, constraint-based structure learning algorithms such as the PC algorithm \citep{Spirtes91pc} are guaranteed to find the correct graph if certain assumptions hold. To make sure that a correct graph exists, one assumes that the data was generated from a distribution that is faithful to an undirected graph or a DAG for Markov and Bayesian networks, respectively. Another crucial assumption is that conditional independence tests always yield correct results. In other words, one assumes that there is infinite amount of data or one equivalently has access to a conditional independence oracle.

In practice, conditional independence is evaluated with statistical tests on data. This leads to occasional errors and theoretical guarantees do not hold in practice. Motivated by this, we study theoretical properties of structure learning in a relaxed setting when we only have access to an unreliable oracle, that is, the oracle might make a small bounded number of errors. We are then interested in how many errors the oracle can be allowed to make if we want to be guaranteed to recover the hidden graph and how this affects the computational complexity of structure learning.

Our first observation is that the number of allowed errors depends on the structure of the hidden graph. To this end, we introduce a property called {\em $k$-identifiability}. If a graph (in case of  Markov networks) or a Markov equivalence class (in case of Bayesian networks) is $k$-identifiable then it can be identified uniquely given any oracle which makes at most $k$ errors. We show that Markov networks with small maximum pairwise connectivity \citep{Korhonen2024} are $k$-identifiable even when $k$ is exponential in the number of vertices (\Cref{thm:identifiability_markov}). The same does not hold for Bayesian networks; we rule out the possibility of bounding the number of errors by many graph parameters that have turned out to be useful in other situations (\Cref{sec:bn-identifiability}). We also notice that computing exact $k$ seems difficult and provide results for graphs whose structure is a chain (\Cref{lem:mn-nearest} and \Cref{thm:closest-MEC-for-path-graphs}).

Finally, we study structure learning with an unreliable oracle. We show that the structure of a Markov network can be found in time $n^{2k + O(1)} \cdot 2^n$ where $n$ is the number of vertices (Theorem~\ref{thm:learning_markov}) and the structure of a Bayesian network in time $n^{2k + O(1)} 2^{n(k + O(1))}$ (Theorem~\ref{thm:learning_bayes}). In contrast, Markov networks can be learned with $O(n^2)$ queries and in polynomial time with an error-free oracle. Constraint-based structure learning for Bayesian networks is NP-hard even with an error-free oracle \citep{Bouckaert94}. We further show that, in the worst case, one needs to query all the possible ${n \choose 2} 2^{n - 2}$ conditional independence tests even when the oracle makes at most $1$ error and we are promised that the hidden graph is one of two given candidates (Theorems~\ref{thm:queries_markov} and \ref{thm:queries_bayes}).

{\bf Related work.} The presence of errors has lead to development of variants of the PC algorithms, such that robust PC \citep{Kalisch2008} and conservative PC \citep{10.5555/3020419.3020468}, which are better suitable to handling errors. Another approach is to pose structure learning as a constraint satisfaction problem and to minimise (weighted) sum of violated test results \citep{10.5555/3020751.3020787}.

Recently, \citet{faller2025differentnotionsredundancyconditionalindependencebased} have studied error correction in structure learning. Their work concentrates on what kind of redundant tests are most useful.

\section{Preliminaries}

We denote an undirected graph by $G = (V, E)$ and a directed graph by $D = (V, A)$, where $V$, $E$, and $A$ are the sets of vertices, undirected edges of the form $\{u, v\} \subseteq V$, and directed edges or arcs of the form $(u, v) \in V \times V$, respectively. The endpoints of both types of edges are $u$ and $v$. 
A $v_1$--$v_\ell$ path is a sequence of distinct vertices $v_1, v_2, \dots, v_\ell \in V$ such that the graph has an edge of the form $\{v_i, v_{i+1}\}$, $(v_i, v_{i+1})$, or $(v_{i+1}, v_i)$ for all $i \in \{1, 2, \dots, \ell - 1\}$. Vertices $v_2, v_3, \dots, v_{\ell - 1}$ are the internal vertices of the path. The \emph{skeleton} of a directed graph is its underlying undirected graph. Throughout this paper, we use $n = |V|$ to denote the number of vertices.

In probabilistic graphical modeling, both \emph{Markov} and \emph{Bayesian networks} represent a set of variables as vertices and capture their conditional independencies with the structure of the network that is either an undirected graph or a DAG, respectively. 
For a detailed overview of probabilistic graphical modeling, see the textbook of \cite{Koller09}.
Since we focus only structure learning and not parameter estimation, we will---with a mild abuse of terminology---use the terms Markov network and undirected graph interchangeably as well as the terms Bayesian network and DAG.
To learn the structure called the \emph{hidden graph} that corresponds to the data-generating process, one of the main approaches is to perform conditional independence queries for the variables and then find a graph that matches them the best. A Markov network~$G$ encodes the claim that variables $u$ and $v$ are conditionally independent given $S \subseteq V \setminus \{u, v\}$ if every $u$--$v$ path has an internal vertex from $S$, that is, they are \emph{separated} by $S$. 

For a Bayesian network $D$, a more involved notion of \emph{d-separation} is needed to encode the conditional independencies. A \emph{collider} on an $u$--$v$ path $v_1, v_2, v_3, \dots, v_{\ell - 1}, v_\ell$ with $u = v_1$ and $v = v_\ell$ is a vertex $v_i$ with $i \in \{2, 3, \dots, \ell - 1\}$ such that the graph has edges $(v_{i-1}, v_i)$ and $(v_{i+1}, v_i)$. Variables $u$ and $v$ are then d-separated by $S$ if for every $u$--$v$ path $v_1, v_2, v_3, \dots, v_{\ell - 1}, v_\ell$
there is $i \in \{2, 3, \dots, \ell - 1\}$ such that either
\begin{enumerate}[label=(\roman*)]
    \item $v_i$ is not a collider on the path and $v_i \in S$; or
    \item $v_i$ is a collider on the path and neither $v_i$ nor any of its descendants are in $S$, where $v_j$ is a descendant of $v_i$ if there exists a directed path from $v_i$ to $v_j$. 
\end{enumerate}
Notably, if there is an undirected or a directed edge between two vertices, then they are not (d-)separated given any subset of vertices.

Markov and Bayesian networks that encode the same set of conditional independencies are said to belong to the same \emph{Markov equivalence class} (MEC).
Two Markov networks with different structures are never Markov equivalent, but this is not always the case for Bayesian networks because of d-separabilities. For Bayesian networks, Markov equivalence is characterized by the two networks having the same underlying undirected graph and the same set of \emph{v-structures}, that is, triples of vertices $u$, $v$, and $w$ with the graph having edges $(u, v)$ and $(w, v)$ but neither $(u, w)$ nor $(w, u)$. 

For variables $u, v \in V$ and $S \subseteq V$, we will write $u \ind v \mid S$ if $u$ and $v$ are conditionally independent given $S$. Further, we write $u \ind[G] v \mid S$ if $S$ separates them in a Markov network $G$ and $u \ind[D] v \mid S$ if $S$ d-separates them in a Bayesian network $D$. For conditional dependence and lack of separation, we use the symbol $\not\ind$. We assume that the underlying probability distribution of the variables is \emph{faithful} to some Markov network $G$, that is, $u \ind v \mid S$ if and only if $u \ind[G] v \mid S$. Similarly, we assume faithfulness to some Bayesian network $D$ for the results relating to Bayesian networks.

We further assume that we have access to a conditional independence (CI) oracle $\mathcal{Q}$ which returns \verb|Independent| to a query $\mathcal{Q}(u, v, S)$ if $u$ and $v$ are conditionally independent given $S$ and \verb|Not independent| otherwise, for $u, v \in V$ and $S\subseteq V \setminus \{u, v\}$. 

Suppose we have queried the conditional independencies from an oracle $\mathcal{Q}$ for all variables $u, v \in V$ and $S \subseteq V$, but oracle may have made up to $k$ errors. These errors may be arbitrary and possibly even adversarial. We define the \emph{(d-)separation distance} of a graph to $\mathcal{Q}$ as the number of conditional independence queries whose outcome differs from the one implied by the structure of the graph. The (d-)separation distance of a MEC to $\mathcal{Q}$ is the (d-)separation distance of any of its members. Distances between two Markov networks, two Bayesian networks, and two MECs are defined analogously. In particular, we write $d(M,M')$ for the distance of two MECs $M$ and $M'$. We then study the impact of the number of errors on whether we can identify the MEC of the hidden graph uniquely and in that case output a member from the MEC. For Markov networks, that undirected graph is also unique.

\mnsl{}\\
\emph{Input:} Vertex set $V$, CI oracle $\mathcal{Q}$, error bound $k$\\
\emph{Output:} Output a unique undirected graph within separation distance $k$ to $\mathcal{Q}$ or that none exists or it is not unique

\bnsl{}\\
\emph{Input:} Vertex set $V$, CI oracle $\mathcal{Q}$, error bound $k$\\
\emph{Output:} Output any DAG from a unique MEC within d-separation distance $k$ to $\mathcal{Q}$ or that none exists or it is not unique

For conciseness, we abbreviate the problems as \mabbr{} and \babbr{}, respectively. Note that if $k=0$, then the problem reduces to constraint-based structure learning with an error-free oracle. Graphs from distinct Markov equivalence classes differ for at least (d-)separation query, so a unique solution can always be identified with some number of queries in this case.

We say that an undirected graph $G$ is {\em $k$-identifiable} if separation distance from every other graph to $G$ is at least $2k + 1$. In other words, \mabbr{} with any oracle that makes at most $k$ errors returns \verb|Unique| if the hidden graph is $G$. Analogously, a DAG $D$ is $k$-identifiable by an oracle that makes $k$ errors if d-separation distance from any other MEC to the MEC of $D$ is at least $2k+1$. 

\section{$k$-Identifiability}

All undirected graphs and MECs are $0$-identifiable. That is, given our assumption that the distribution is faithful to an undirecred graph or DAG, they can be uniquely identified by a CI oracle that do not make any errors. We also notice that when $k$ increases then fewer graphs will be $k$-identifiable. Thus, it is interesting to study what is the largest $k$ with which a specific graph is $k$-identifiable. To compute this, it is enough to find the nearest neighbor of the given graph or MEC with respect to separation or d-separation distance. If the distance is $d$ then the graph or MEC is $(d/2 - 1)$-identifiable.

In this section, we study how the structure of the hidden graph affects $k$-identifiability. First, we ask whether we can derive bounds for the $k$ as a function of graph parameters. Then, we study how we can compute the largest $k$ for a graph.

\subsection{Parameterised bounds}

\subsubsection{Markov networks}

Let us first analyse $k$-identifiablility in Markov networks. First observation is that there are graphs which are not $k$-identifiable for any $k>0$. Consider a complete graph $G_1 = (V, E_1)$. Let $G_2= (V, E_2)$ be an almost complete graph where $E_2 = E_1\setminus \{v, u\}$. Clearly, no pair of vertices is separated in $G_1$ given any conditioning set. For $G_2$, there is only one conditioning set $S$ such that the $u$ and $v$ are separated given $S$. Thus, the separation distance between $G_1$ and $G_2$ is $1$. This implies that $G_1$ is not $k$-identifiable for any $k>0$.

Next, we will show that there are $k$-identifiable graphs where $k$ is exponential in $n$. We start with a definition. {\em Maximum pairwise connectivity} $\kappa$ \citep{Korhonen2024} of an undirected graph $G$ is defined as follows. For two vertices $u, v\in V$, let $\kappa(u, v, G)$ be the number of vertex-disjoint paths, each having at least one internal vertex, between $u$ and $v$. Now the maximum pairwise connectivity of graph $G$ is the maximum of $\kappa(u, v, G)$ over all pairs $u$ and $v$. That is, $\kappa (G) = \max_{u, v\in V} \kappa (u, v, G)$.

\begin{lemma} \label{lemma:markov_kappa}
    Let $\kappa(G)$ be the maximum pairwise connectivity of $G$. Given two distinct undirected graphs $G_1 = (V, E_1)$ and $G_2 = (V, E_2)$, there are at least $2^{n - 2 - \kappa(G_1)}$ separation statements with different outputs for $G_1$ and $G_2$. 
\end{lemma}
\begin{proof}
    Since $G_1$ and $G_2$ are distinct, either $G_2$ has an edge $uv$ that $G_1$ does not have or $G_2$ is a proper subgraph of $G_1$. In the latter case $\kappa(G_2) \le \kappa(G_1)$ and $G_1$ has an edge $uv$ that $G_2$ does not have, so it suffices to prove the former case.

    Because $u$ and $v$ are not adjacent in $G_1$, they are separated by some separator $S$ of size at most $\kappa(G_1)$. Any superset $T \subseteq V \setminus \{u, v\}$ of $S$ also separates them, and there are $2^{n - 2 - \kappa(G_1)}$ such sets. In contrast, these vertices are never separated in $G_2$. 
\end{proof}

The following theorem follows directly from Lemma~\ref{lemma:markov_kappa}.

\begin{theorem} \label{thm:identifiability_markov}
    Let $\kappa(G)$ be the maximum pairwise connectivity of $G$. Then $G$ is $(2^{n - \kappa(G) -3} - 1)$-identifiable.
\end{theorem}

We note that the above result is a lower bound. That is, the graph $G$ is guaranteed to be $(2^{n - \kappa(G) -3} - 1)$-identifiable but it may be $k$-identifiable with larger values of $k$.

Theorem~\ref{thm:identifiability_markov} shows that the lower bound for $k$ grows exponentially in $n$. However, we should note that the number of different tests is ${n \choose 2} \times 2^{n-1}$ and thus the theorem does not imply that the probability of erroneous tests goes down or even stays the same when $n$ grows.

\subsubsection{Bayesian networks}\label{sec:bn-identifiability}

We would like to provide similar parameterised bounds for Bayesian networks. However, v-structures complicate learning and we can rule out many natural graph parameters.

Let us consider four hidden graphs on vertex set $V = \{v_1, v_2, \ldots, v_n\}$. The first graph $D_\emptyset = (V, \emptyset)$ is an empty graph. The second graph is $D_1 = (V, E_1)$ where the arc set $A_1$ consists of the following arcs: $(v_1, v_3)$, $(v_2, v_3)$, and $(v_i, v_{i+1})$ for $3\leq i \leq n-1$. The third graph is $D_2 = (V, E_2)$ where the graph consists of $r$ cliques of size $n/r$ with a constant $r\geq 2$. Finally, the fourth graph is $D_C = (V, E_C)$ is a complete graph. That is, one of the arcs $(v_i, v_j)$ or $(v_j, v_i)$ is in $E_C$ for every $i$ and $j$ with $i \neq j$.

First, let us analyse $k$-identifiability of the graph is $D_\emptyset$, $D_1$, $D_2$ or $D_C$. Suppose the hidden graph is $D_\emptyset$. Now for every $v_i$, $v_j$ and $S$,  $v_i$ and $v_j$ are d-separated by $S$. On the other hand, $v_i$ and $v_j$ are never d-separated in any graph with an arc $(v_i, v_j)$ yielding at least $2^{n-2}$ different d-separation statements compared to $D_\emptyset$. Thus, $D_\emptyset$ is $(2^{n-3} - 1)$-identifiable.

Now suppose that the hidden graph is $D_1$. Consider a graph $D' = (V, A')$ with an arc set $A'$ consists of the following arcs: $(v_1, v_2)$,$(v_1, v_3)$, $(v_2, v_3)$, and $(v_i, v_{i+1})$ for $3\leq i \leq n-1$. DAGs $D_1$ and $D'$ are illustrated in Figure~\ref{fig:dags}. Vertices $v_1$ and $v_2$ are d-separated by the empty set in $D_1$ but not in $D'$. All other d-separations are the same. Thus, the d-separation distance between $D_1$ and $D'$ is $1$. This implies that $D_1$ is not $k$-identifiable for any $k>0$.

\begin{figure}[!htp]
\centering
\begin{tabular}{cc}
     \includegraphics[width=0.3\linewidth]{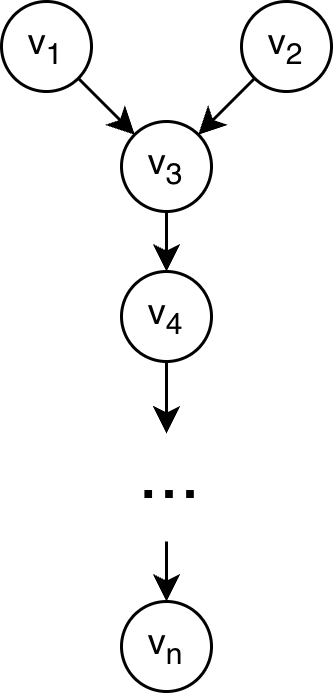} &   \includegraphics[width=0.3\linewidth]{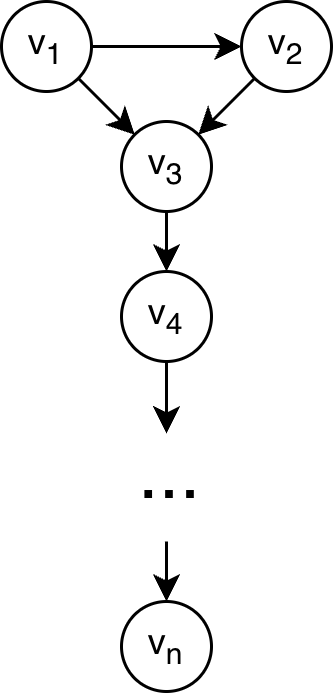}\\
     (a) $D_1$ & (b) $D'$\\ 
\end{tabular}
\caption{Two sparse graphs whose d-separation disctance is 1.} \label{fig:dags}
\end{figure}

Next, let us consider $D_2$. Suppose that a graph $D''$ contains an arc $(v_i, v_j)$ that is not present in $D_2$. This implies that $v_i$ and $v_j$ are in different cliques in $D_2$ and are thus d-separated by every $S$ in $D_2$. However, they are never d-separated in $D''$ which means that there are at least $2^{n-2}$ different d-separation statements between $D_2$ and $D''$. Now suppose that $D'''$ has no arcs that are not present in $D_2$. Let us say that $D'''$ does not not have arc $(v_i, v_j)$ which is present in $D_2$; Without loss of generality, let us assume that $v_i$ and $v_j$ are in clique $1$ in $D_2$. Clearly, $v_i$ and $v_j$ are never d-separated in $D_2$. Vertices $v_i$ and $v_j$ are d-separated by at least one subset $S$ of vertices in clique 1. Furthermore, they are d-separated by union  of $S$ and any subset of the vertices in other cliques. There are $2^{n - n/r}$ such sets. Thus, the d-separation distance to the nearest neighbor of $D_2$ is at least $2^{n - n/r}$ and thus $D_2$ is $(2^{n - n/r -1} -1)$-identifiable.

Finally, let us look at the complete graph $D_C$. None of the pairs $v_i$ and $v_j$ are d-separated by any $S$ in $D_C$. Let us compare $D_C$ to an almost complete graph that does not have an arc between $v_i$ and $v_j$ (but has an arc between every other pair). Now there is exactly one set $S$ such that $v_i$ and $v_j$ are d-separated by $S$. Therefore, $D_C$ is not $k$-identifiable with any $k>0$. This is also what \citet{faller2025differentnotionsredundancyconditionalindependencebased} have observed.

To find a lower bound for $k$-identifiability, we want to find a graph parameter $p (D)$ and a monotone function $f(p)$ such that if $D$ is $k$-identifiable than $f(p(D), n) < k$ (in case of an upper bound $f(p(D), n) > k $).

In summary, if the hidden graph is either $D_\emptyset$ or $D_2$ then we can allow the oracle to make an exponential number of errors but if the hidden graph is either $D_1$ or $D_C$ then the oracle cannot make any errors.  This implies that no graph parameter $p$ such that $p(D_\emptyset) < p(D_1) < p(D_2) < p(D_C)$ can give upper or lower bounds for $k$ (also $p(D_\emptyset) < p(D') < p(D_2) < p(D_C)$ disqualifies the parameter). 

This means that many reasonable parameters such as number of arcs, maximum pairwise connectivity, treewidth, or size of the maximum undirected clique are not going to give such bounds.

Unlike with Markov networks, maximum pairwise connectivity does not give bounds for $k$-identifiability in Bayesian networks. First, maximum pairwise connectivity is defined as a property of an undirected graph. There is no well-defined way to define maximum pairwise connectivity of a DAG so we consider two alternatives: Maximum pairwise connectivity of the skeleton of a DAG or maximum pairwise connectivity of the moral graph \footnote{A moral graph of a directed acyclic graph $D = (V, A)$ is an undirected graph $G = (V, E)$ such 1) if $uv\in A$ then $\{u, v\} \in E$ and 2) if $uv\in A$ and $wv\in A$ then $\{u, w\}\in E$.} of the DAG. Regardless of the definition, maximum pairwise connectivities of $D_\emptyset$, $D_1$, $D_2$ and $D_C$ are $0$, $1$, $n/r - 2$ and $n-2$, respectively.

We conducted a small empirical study to see how sparsity of $k$-indentifiability. To this end, we  generated all DAGs with 5 vertices, choosing one DAG to represent every MEC and computed the distance to the nearest MEC for every MEC. Results are summarised in Table~\ref{tab:distance_distribution}. The results show that, on average sparser networks tend to have higher $k$. However, the number of arcs does not give strict bounds which can be demonstrated by comparing number of arcs in $D_\emptyset$, $D_1$, $D_2$ and $D_C$.

\begin{table}[]
    \centering
        \caption{Distances to the nearest neighbor in MECs with $5$ vertices. MECs are grouped by the number of edges and the table shows minimum, mean and maximum distance to the nearest neighbor MECs with a specific number of edges.}
    \label{tab:distance_distribution}
    \begin{tabular}{ccccc}
        \toprule
        Nr of edges & Nr of MECs & Min & Mean & Max\\
        \midrule
        0 & 1 & 8 & 8 & 8 \\
        1 & 10&  8 & 8 & 8 \\
        2 & 75 & 4 & 4.8 & 8\\
        3 & 350 & 2 & 4.7 & 6\\
        4 & 1120 & 1 & 3.8 & 7 \\
        5 & 2130 & 1 & 2.3 & 6 \\
        6 & 2595 & 1 & 1.5 & 4 \\
        7 & 1730 & 1 & 1.1 & 3 \\
        8 & 690 & 1 & 1 & 1 \\
        9 & 80 & 1 & 1 & 1\\
        10 & 1 & 1 & 1 & 1\\
        \bottomrule
    \end{tabular}
\end{table}

Treewidth is often used to measure complexity of a Bayesian network as inference is tractable in low treewidth networks.\footnote{Treewidth is defined as follows. A tree-decomposition of an undirected graph $G = (V, E)$ is a pair $(\mathcal{X}, T)$ where $\mathcal{X} = \{X_1, \ldots X_m\}$ is a collection of subsets of $V$ and $T$ is a tree on $\mathcal{X}$ such that 1) $\bigcup_{i = 1, \ldots m} X_i = V$, 2) for all edges $\{e, v\} \in E$, there exists $i$ such that $u\in X_i$ and $v\in X_i$, and 3) for all $i$, $j$ and $k$, if $X_j$ is on the unique path between $X_i$ and $X_k$ in $T$ then $X_i \cap X_k \subseteq X_j$. The width of a tree-decomposition is $\max_{i = 1, \ldots m} |X_i| - 1$. {\em Treewidth} of graph $G$ is the smallest width over all possible tree-decompositions. 

 Treewidth of a DAG is defined as treewidth of its moral graph.} 
We observe that treewidths of $D_\emptyset$, $D_1$, $D_2$ and $D_C$ are $0$, $2$, $n/r-1$ and $n-1$, respectively. Thus, treewidth does not provide meaningful bounds for $k$-identifiability.

It has been shown that a parameter called size of the maximum undirected clique\footnote{Size of the largest undirected clique is defined as follows. An essential graph of DAG $D=(V, A)$ is a partially directed graph whose skeleton is the same as the skeleton $D$. The essential graph has an arc $uv$ if the arc is between $u$ and $v$ is directed towards $v$ in every DAG in the equivalence class of $D$. Otherwise, the essential graph an undirected edge $\{u, v\}$.
 
A clique in an undirected graph is a set of vertices such that all vertices are adjacent to each other. The {\em size of the maximum undirected clique} is the size of largest clique after all directed arcs are removed from the essential graph.} is essential in the testing problem where one tries to test whether the distribution is faithful to a given graph or not\citep{zhang2024membership,10.5555/3692070.3693904}. 
However, it cannot be used here because sizes of maximum undirected cliques for $D_\emptyset$, $D'$, $D_2$ and $D_C$ are $0$, $1$, $n/r$ and $n$, respectively.

In conclusion, it seems difficult to find parameterised bounds for Bayesian networks.

\subsection{Finding nearest neighbors}

Ideally, we would like to compute the largest $k$ for an arbitrary DAG $D$. In principle, this can be done by generating all possible Markov equivalence classes and computing their d-separation distances to $D$. However, this is not computationally feasible in practice. Thus, we ask whether there exists a more efficient way to find the nearest neighbor.

We hypothesize that the MEC that is closest to $D$ with respect to d-separation distance is also close to $D$ ``structurewise''. In other words, we hypothesize that $D$ can be transformed to its nearest neighbor by a simple local operation such as adding an arc, removing an arc or reversing and arc. However, it seems that proving this is highly non-trivial. Here, we study nearest neighbors in a very restricted class of graphs, namely chains.

Recall that a chain is a graph where its vertices have some order $v_1, v_2, \dots, v_n$ such that the edges of the underlying undirected graph are exactly $\{v_1, v_2\}, \{v_2, v_3\}, \dots, \{v_{n-1}, v_n\}$.

\begin{theorem}\label{lem:mn-nearest}
    Nearest neighbor of a Markov network~$G$ that is a chain is obtained by adding an edge for $n \ge 3$ and is at a distance $2^{n-2} - 1$ from $G$.
\end{theorem}
\begin{proof}
    Let $u$ be a leaf of the chain, $v$ the unique vertex that has distance $2$ to $u$ in $G$, and $w$ the unique vertex that has distance $1$ to $u$ in $G$. We start by proving that adding the edge~$uv$ changes the result of $2^{n - 2} - 1$ separation queries: Previously, $u$ was separated from each unique vertex $x$ whose distance from $u$ in $G$ was $\ell$ by any of the $(2^{\ell-1} - 1)2^{n - \ell - 1}$ conditioning sets that contained at least one of the vertices with distance from $u$ less than $\ell$. To remain separated after adding the edge, they have to contain any of the $2(2^{\ell-2} - 1)2^{n - \ell - 1}$ conditioning sets that contain at least one of the vertices with distance from $u$ more than $1$ and less than $\ell$. In total, this contributes 
    \[ (2^{\ell-1} - 1)2^{n - \ell - 1} - 2(2^{\ell-2} - 1)2^{n - \ell - 1} = 2^{n - \ell - 1} \]
    differing query results, and summing over $\ell=2\dots (n-1)$ yields $2^{n - 2} - 1$ queries whose outcome has changed.

    Observe next that adding at least two edges changes the outcomes of more than $2^{n-2} - 1$ queries even if we add or remove any number of additional edges. If we were to add a new edge $xy$, there is always some vertex $z$ between them in the chain such that any conditioning set containing $z$ separates them in $G$ but not after adding the edge. This then causes at least $2^{n-3}$ separation queries for the pair of vertices $x$ and $y$ to change their outcome irrevocably. If we then were to add another edge, at least one of its endpoints differs from $x$ and $y$, and thus changes at least $2^{n-3}$ other separation query outcomes. In total, this incurs at least $2^{n-2}$ changes.

    The remaining case to consider is thus when we do some number of removals and then possibly add an edge. Start by noting that removing any edge $ab$ of $G$ changes the result of at least $2^{n-2}$ separation queries: previously, $a$ and $b$ were inseparable for all $2^{n-2}$ conditioning sets, but now they are separated for all of them since they are in distinct components. To prevent at least some of the changed outcomes, we also have to add an edge $uv$ somewhere.
    Further, the graph must become connected afterwards or otherwise we again get at least $2^{n-2}$ changed outcomes: if the graph is disconnected, there must be an edge $ab$ of $G$ such that $a$ and $b$ are in distinct components and always separated. Since we can only add one edge, we can also only remove one edge to retain connectedness.
    
\end{proof}

Next, we study Bayesian networks whose skeleton is a chain. Before that, we need to introduce some definitions.

\input{definitions}

We also have the following characterization for a Bayesian network with a path skeleton.
\begin{lemma}
\label{lem:d-separation-in-paths}
Let $D$ be a Bayesian network whose skeleton is a path on $V=\{v_1,\ldots,v_n\}$, and  $\mathcal C$ be the set of colliders in $D$. For $1\le i<j\le n$, let $S_{i,j}$ be the set of internal vertices in the path between $v_i$ and $v_j$ and let $\mathcal C_{i,j} = \mathcal C\cap S_{i,j}$. Then for any $Z\subseteq V\setminus\{v_i,v_j\}$, $v_i$ and $v_j$ are d-connected given $Z$ iff $Z\cap S_{i,j}=\mathcal C_{i,j}$.
\end{lemma}
\begin{proof}
As discussed in the introduction, $v_i$ and $v_j$ are d-separated by $Z \subseteq V\setminus \{v_i, v_j\}$ if every $v_i$--$v_j$ path $u_1, u_2, \ldots, u_{l-1},  u_l$ there is an $a \in \{2, 3, \ldots, l-1\}$ such that either $u_a \in Z$ and $D$ lacks one of the edges $(u_{a-1}, u_a)$ or $(u_{a+1}, u_a)$, or neither $u_a$ nor any of its descendant is in $Z$ and $D$ has edges $(u_{a-1}, u_a)$ and $(u_{a+1}, u_a)$.   
A path skeleton contains a unique path between any two vertices, and in such a Bayesian network, every collider has no descendants. Hence, to $d$-connect $v_i$ and $v_j$, along the unique path between $v_i$ and $v_j$, non-colliders must be excluded from $Z$, and colliders must belong to $Z$, which is equivalent to $Z\cap S_{i,j}=\mathcal C_{i,j}$.
\end{proof}

The proof of the following theorem is rather involved, and thus we present only a proof sketch here. The complete proof can be found in the appendix.

\begin{theorem}
        \label{thm:closest-MEC-for-path-graphs}
        For a Bayesian network $D$ with a chain skeleton on $n \geq 3$ vertices 
$V = \{v_1, v_2, \ldots, v_n\}$, and the family of graphs whose skeleton is a chain $\mathcal{F}$ 
with vertex set $V$, we have $d_{\mathcal{F}}(D) = 2^{n-1} - 2$. 
\end{theorem}

\begin{proof}[Proof sketch of Theorem~\ref{thm:closest-MEC-for-path-graphs}:]
We first construct a DAG $D'$ by swapping the positions of $v_1$ and $v_2$ such that the set of colliders of $D'$ is a subset of the colliders of $D$. 
If $v_2$ is not a collider in $D$, then the sets of colliders in $D$ and $D'$ are the same. 
If $v_2$ is a collider in $D$, then all colliders of $D$ except $v_2$ are colliders in $D'$. 
Since the skeletons of $D$ and $D'$ are different, they are not equivalent. 
The distance between $D$ and $D'$ is $2^{n-1} - 2$ (the proof is provided in the Appendix; Theorem~\ref{thm:closest-MEC-for-path-graphs-complete}).

To prove that $D'$ is closest to $D$, we show that for every other Bayesian network $D''$ whose skeleton is a chain on $n$ vertices and that is not equivalent to $D$, the distance between $D$ and $D''$ is at least $2^{n-1} - 2$. 
The skeleton of $D''$ is either the same as that of $D$ or different from it. 
In both cases, we show that the distance between $D$ and $D''$ is at least $2^{n-1} - 2$. We use \cref{lem:d-separation-in-paths} to compare the d-connectivity in the Bayesian networks $D, D'$ and $D''$. 
\end{proof}

We note that Theorem~\ref{def:closest-MEC-with-specific-graph-structure} only considers neighbors of $D$ whose skeleton is a chain. Specifically, the theorem does not state that this is the nearest neighbor, and it is possible that there are more nearby neighbor that are for example obtained by adding an edge.

The used proof techniques unfortunately do not easily generalize for arbitrary graphs or even very restricted graph classes such as trees. Based on small-scale experiments on graphs with 5 vertices, it seems that the closest graph always differs by a single edge, which we formalize in the following hypotheses. 
\begin{conjecture}
    The nearest neighbor of any Markov network $G$ is obtained by adding or removing a single edge.
\end{conjecture}
\begin{conjecture}
    The nearest neighbor of any Bayesian network $D$ is obtained by adding, reversing, or removing a single edge.
\end{conjecture}

\section{Learning with an unreliable oracle}

In this section, we study how to utilize the query results to learn the hidden graph.
For both \mabbr and \babbr, the simplest algorithm is to simply enumerate all possible hidden graphs and pick the one minimizing the (d-)separation distance to the oracle. There are $2^{\binom{n}{2}}$ distinct undirected graphs on $n$ vertices and $2^{\binom{n}{2}} \cdot n!$ DAGs up to exponential factors \citep{Stanley06}. For each graph, we can test each of the $\binom{n}{2} \cdot 2^{n - 2}$ (d-)separabilities in polynomial time. 

If the upper bound $k$ for the number for the errors is small, the problems are solvable faster.
For a Markov network, initialize a graph such that we have an edge $uv$ if $u$ and $v$ are not separated by $V \setminus \{u, v\}$. If more than $k$ queries are inconsistent with the graph, remove or add one edge $uv$. Then, the query result for the separability of $u$ and $v$ conditioned by $V \setminus \{u, v\}$ is inconsistent with the resulting graph. Since nothing can be gained by adding and removing the same edge, we repeat the process at most $k$ times. The size of the search tree is then at most $\binom{n}{2}^k = O(n^{2k})$ and in each vertex of the tree we check all $\binom{n}{2} \cdot 2^{n - 2}$ query results once, resulting in total time complexity $n^{2k + O(1)} \cdot 2^{n}$.
\begin{theorem} \label{thm:learning_markov}
    \mabbr can be solved in time $n^{2k + O(1)} \cdot 2^{n}$.
\end{theorem}

For Bayesian networks, the previous approach is not applicable because adding edges can introduce cycles. Therefore, we instead enumerate over all $O(n^{2k} 2^{nk})$ groups of up to $k$ tests and flip their results by assuming that they are the erroneous ones. We can then run, for example, the PC algorithm \cite{Spirtes91pc} to construct a candidate for the hidden graph, for which we then assess whether all the tests are faithful. In total, this requires $n^{2k + O(1)} 2^{n(k + O(1))}$ time.
\begin{theorem} \label{thm:learning_bayes}
    \babbr can be solved in time $n^{2k + O(1)} 2^{n(k + O(1))}$.
\end{theorem}

A natural question is whether one can get rid of the exponential dependence in $n$. Unfortunately, this turns out to be unconditionally impossible.
\begin{theorem} \label{thm:queries_markov}
    One cannot solve \mabbr without performing $\binom{n}{2}2^{n-2}$ queries in the worst-case even if $k = 1$ and we are promised that the hidden graph is either some given graph $G_1$ or $G_2$. 
\end{theorem}
\begin{proof}
    Consider an undirected graph $G_1 = (V, E)$ constructed as follows: pick arbitrarily two distinct vertices~$u, v \in V$ and add edges $\{u, w\}$ and $\{v, w\}$ for all other vertices $w \in V \setminus \{u, v\}$. Let $G_2 = (V, E \cup \{u, v\})$ and note that $x \ind[G_1] y \mid S$ if and only if $x \ind[G_2] y \mid S$ for all $x, y \in V$ and $S \subseteq V \setminus \{u, v\}$ except for $x = u$, $y = v$, and $S = V \setminus \{u, v\}$.

    Let $k = 1$ and suppose we have performed the query $q \coloneqq \mathcal{Q}(u, v, V \setminus \{u, v\})$ and received the output \verb|Not independent|. This query has to be made, since it is the only one that can distinguish between $G_1$ and $G_2$ if the output is correct. The correct output of $\mathcal{Q}$ is known beforehand for all other queries by construction, so we can identify if $\mathcal{Q}$ makes an error if we perform any other queries. 
    Suppose we detect such an error. Then, the hidden graph has to be $G_2$. On the other hand, if we detect no errors, then the graph could be either $G_1$ or $G_2$, since we cannot decide whether the output for $q$ was erroneous. If the output for $q$ was not an error, an adversarial oracle $\mathcal{Q}$ could decide to give the correct output for the first $\binom{n}{2}2^{n-2} - 1$ queries and then on the last possible query make an error that can be detected. Alternatively, the oracle could decide to give the correct output for all queries. Thus, to distinguish between the possibilities that there is a unique solution or multiple possible solutions, one has to perform all possible the queries.
\end{proof}

The claim can be proven analogously for \babbr with the DAGs from Figure~\ref{fig:dags}.
\begin{theorem} \label{thm:queries_bayes}
    One cannot solve \babbr without performing $\binom{n}{2}2^{n-2}$ queries in the worst-case even if $k = 1$ and we are promised that the hidden graph is either some given graph $D_1$ or $D_2$. 
\end{theorem}

Note that the proof relies on there being two graphs that are close to each other with respect to (d-)separations. If there are no other graphs that are close to the hidden graph and the oracle can make only a small number of errors, then identifying the hidden graph is possible with significantly smaller number of queries.

These results also provide sharp contrast to the case with $k = 0$ for \mabbr, where the problem can always be solved with $O(n^2)$ queries. Further, if $k = 0$ and we are given a promise that the hidden graph is one of two graphs, then both MNSL and BNSL can be solved with one query.

\section{Concluding Remarks}

We initiated the study of structure learning for Markov networks and Bayesian networks with a CI oracle that may make a limited number of errors. 
We identified structural properties of the hidden graph that make it uniquely identifiable in the presence of errors.
On the other hand, the possibility of even a single error may necessitate performing all possible conditional independence queries in the worst case.
This highlights the need for structure learning algorithms that are able to exploit these structural properties to avoid having to conduct all the tests when there are only few errors.

Having a bounded number of errors also enables error correction, which the present work did not explore. For example, separability in Markov networks behaves monotonically with respect to subset inclusion, meaning that if the oracle claims that $S$ does not separate $u$ and $v$ but over $k$ of its subsets do, then the claim has to be an error. It however remains unclear whether there are more involved error correction schemes that are able to fix more ``subtle'' errors. Further, the errors tend not to be independent or uniformly distributed in practice, providing alternative directions for improving the algorithms.

\section*{Acknowledgements}

This work was supported by L. Meltzers Høyskolefond (PP). 

\bibliography{main}

\newpage
\onecolumn
\title{Learning Bayesian and Markov Networks with an Unreliable Oracle\\(Supplementary Material)}
\maketitle
\appendix

\section{Complete Proofs}

\input{closest-MEC-path-result}

\end{document}

%% file: definitions.tex
\begin{definition}[\textit{Table Function for a Bayesian Network}]
    \label{def:table-function-of-an-MEC}
    Let $D$ be a Bayesian network defined over a set of random variables $V = \{v_1, v_2, \ldots, v_n\}$. A table function for $D$ is a function 
    $T_D : \{(v_i, v_j, Z) : v_i, v_j \in V, v_i \neq v_j, \text{ and } Z \subseteq V \setminus \{v_i, v_j\}\} \rightarrow \{0, 1\}$ 
    such that $T_D(v_i, v_j, Z) = 1$ if and only if $v_i$ is d-connected \footnote{$v_i$ and $v_j$ are d-connected given $Z$ if they are not d-separated by $Z$.} with $v_j$ given $Z$ in $D$.
\end{definition}

\begin{definition}[distance between two Bayesian Networks]
    \label{def:distance-between-two-MECs}
    Let $D_1$ and $D_2$ be two Bayesian networks defined over the same set of random variables $V = \{v_1, v_2, \ldots, v_n\}$. The distance between $D_1$ and $D_2$, denoted by $d(D_1, D_2)$, is the number of triplets $(v_i, v_j, Z)$ with $1 \leq i < j \leq n$ and $Z \subseteq V \setminus \{v_i, v_j\}$ for which table functions of $D_1$ and $D_2$ differ.  More formally,
    $d(D_1, D_2) = \big|\{(u_i, u_j, Z) : 1 \leq i < j \leq n,  Z \subseteq V \setminus \{u_i, u_j\},\, T_{D_1}(u_i, u_j, Z) \neq T_{D_2}(u_i, u_j, Z)\}\big|$.
\end{definition}

The distance between two Markov equivalent Bayesian networks is always zero, as they entail the same set of conditional independence relations.  

\begin{definition}[\emph{Closest Bayesian Network}]
\label{def:closest-MEC-with-specific-graph-structure}
Let $V$ be a set of random variables, $\mathcal{F}$ be a family of graphs, and $D$ be a Bayesian network defined over $V$ such that the skeleton of $D$ belongs to $\mathcal{F}$. 
A closest Bayesian network to $D$ in $\mathcal{F}$ is a Bayesian network $D_c$ such that $D_c$ is not Markov equivalent to $D$, the skeleton of $D_c$ belongs to $\mathcal{F}$, and for all Bayesian networks $D'$ such that $D'$ is not Markov equivalent to $D$ and the skeleton of $D'$ belongs to $\mathcal{F}$, we have $d(D, D_c) \leq d(D, D')$.
We denote the distance between $D$ and a closest Bayesian network to $D$ in $\mathcal{F}$ by $d_{\mathcal{F}}(D)$.
\end{definition}

\cbn{}\\
\emph{Input:} A family of skeletons $\mathcal{F}$ and a Bayesian network $D$ whose skeleton is in $\mathcal{F}$\\
\emph{Output:} d-separation distance to the closest MEC to $D$ in $\mathcal{F}$

%% file: closest-MEC-path-result.tex
\begin{theorem}
\label{thm:closest-MEC-for-path-graphs-complete}
Let $D$ be a Bayesian network whose skeleton is a chain on $n$ nodes, where $n \geq 3$. 
Then there exists a Bayesian network $D'$ whose skeleton is a chain on $n$ nodes such that $D'$ is not equivalent to $D$, and the distance between $D$ and $D'$ is $2^{n-1} - 2$. 
Moreover, there does not exist a Bayesian network $D''$ whose skeleton is a chain on $n$ nodes such that $D''$ is not equivalent to $D$ and the distance between $D$ and $D''$ is less than $2^{n-1} - 2$.
\end{theorem}

\input{fig-path}

\begin{proof}
    We first construct a Bayesian network $D'$ with a path skeleton on $n$ nodes such that $d(D, D') = 2^{n-1} - 2$. 
We then show that $D'$ is closest to $D$. 
We denote by $\mathcal{C}(D)$ the set of colliders in $D$. 
Analogously, $\mathcal{C}(D')$ denotes the set of colliders in $D'$.

    \input{fig-path-example}

    Let the set of random variables on which $D$ is defined be $V = \{v_1,v_2,\ldots,v_n\}$, and let the skeleton of $D$ be as shown in Figure~\ref{fig:path}. We construct a Bayesian network $D'$ whose skeleton is obtained by swapping the positions of $v_1$ and $v_2$ in the $\skel{D}$, as shown in Figure~\ref{fig:path}. 

    If $v_2$ is not a collider in $D$, then the set of colliders in $D'$ is same as the set of colliders in $D$, i.e., $\mathcal{C}(D') = \mathcal{C}(D)$ (see Example~2 of Figure~\ref{fig:path-example}); otherwise, $\mathcal{C}(D') = \mathcal{C}(D) \setminus \{v_2\}$  (see Example~1 of Figure~\ref{fig:path-example}).

    Since the skeletons of $D$ and $D'$ are different, they are not equivalent. We now show that $d(D, D') = 2^{n-1}-2$. For $1 \le i < j \le n$, let $S_{i,j}$ denote the set of nodes between $v_i$ and $v_j$ in $D$, formally
    $S_{i,j} = \{v_k : i<k<j\}$. Similarly, let $S_{i,j}'$ denote the set of nodes between $v_i$ and $v_j$ in $D'$. By construction, for $3 \le i < j \le n$, we have $S_{i,j}=S_{i,j}'$; for $2 < j \le n$,
    $S_{1,j}'=S_{1,j}\setminus\{v_2\}$ and $S_{2,j}'=S_{2,j}\cup\{v_1\}$; and $S_{1,2}=S_{1,2}'=\emptyset$.

    For $1 \le i < j \le n$, let $\mathcal{C}_{i,j}$ be the set of colliders of $D$ in $S_{i,j}$, formally $\mathcal{C}_{i,j} = S_{i,j} \cap \mathcal{C}(D)$, and analogously define $\mathcal{C}_{i,j}' = S_{i,j}' \cap \mathcal{C}(D')$. By construction, if $v_2$ is not a collider in $D$, then $\mathcal{C}_{i,j}=\mathcal{C}_{i,j}'$ for all $i<j$. If $v_2$ is a collider in $D$, then $\mathcal{C}_{i,j}=\mathcal{C}_{i,j}'$ for all $1 < i < j \le n$, and for $1 < j \le n$ we have $\mathcal{C}_{1,j}'=\mathcal{C}_{1,j}\setminus\{v_2\}$. For example,in Figure~\ref{fig:path-example}, in Example~1, $\mathcal{C}_{1,3}=\{v_2\}$, $\mathcal{C}_{2,5}=\{v_4\}$, $\mathcal{C}_{1,3}'=\emptyset$, and $\mathcal{C}_{2,5}'=\{v_4\}$, and in Example~2, $\mathcal{C}_{1,4} = \mathcal{C}_{1,4}' = \{v_3\}$, and $\mathcal{C}_{4,6} = \mathcal{C}_{4,6}' = \{v_5\}$. In the figure, collider nodes are red marked.

    From \Cref{lem:d-separation-in-paths}, for $1\le i<j\le n$ and $Z\subseteq V\setminus\{v_i,v_j\}$, $v_i$ and $v_j$ are d-connected in $D$ given $Z$ iff $Z\cap S_{i,j}=\mathcal{C}_{i,j}$, and are d-connected in $D'$ given $Z$ iff $Z\cap S_{i,j}' =\mathcal{C}_{i,j}'$.

From the construction of $D'$, for $3 \le i < j \le n$ we have $S_{i,j}=S_{i,j}'$ and $\mathcal{C}_{i,j}=\mathcal{C}_{i,j}'$, and also $S_{1,2}=S_{1,2}'=\emptyset$ and $\mathcal{C}_{1,2}=\mathcal{C}_{1,2}'=\emptyset$; thus $T_D(v_i,v_j,Z)=T_{D'}(v_i,v_j,Z)$ for all $3 \le i < j \le n$ and for $(i,j)=(1,2)$, and all $Z \subseteq V \setminus \{v_i, v_j\}$. Therefore, $d(D, D') = |\{(v_i,v_j,Z): 1\leq i \leq 2, 3\le j\le n, Z \subseteq V\setminus\{v_i,v_j\}, \text{ and } T_D(v_i, v_j, Z) \ne T_{D'}(v_i, v_j, Z)\}|$.

Suppose $v_2$ is a collider in $D$.
For $j \in [n]\setminus \{1,2\}$ and $Z \subseteq V\setminus \{v_1, v_j\}$, by the construction of $D'$, if $Z \cap S_{1,j} = \mathcal{C}_{1,j}$ then $Z \cap S_{1,j}' = \mathcal{C}_{1,j}\setminus \{2\} =  \mathcal{C}_{i,j}'$. Therefore, if $v_1$ is d-connected with $v_j$ given $Z$ in $D$ then $v_1$ is d-connected with $v_j$ given $Z$ in $D'$. 

If $v_2 \in Z$ and $Z \cap S_{1,j}' =  \mathcal{C}_{1,j}'$, then $Z \cap S_{1,j} = \mathcal{C}_{1,j}' \cup \{2\} = \mathcal{C}_{1,j}$. Therefore, for $Z \subseteq V\setminus \{v_2, v_j\}$ such that $v_2 \in Z$, if $v_1$ is d-connected with $v_j$ given $Z$ in $D'$ then $v_1$ is d-connected with $v_j$ given $Z$ in $D$. 
And, if $v_2 \notin Z$ and $Z \cap S_{i,j}' =  \mathcal{C}_{i,j}'$ then $Z \cap S_{1,j} = \mathcal{C}_{1,j}' \neq \mathcal{C}_{1,j}$. Therefore, if $v_2 \notin Z$ and $v_1$ is d-connected with $v_j$ given $Z$ in $D'$ then $v_1$ is d-separated with $v_j$ given $Z$ in $D$. Such a $Z$ is of the form $\mathcal{C}_{1,j}' \cup Y$, where $Y \subseteq \{v_{j+1}, v_{j+2} \ldots, v_n\}$.  Therefore, the number of $Z$ for which $v_1$ is d-connected with $v_j$ given $Z$ in $D'$ and $v_1$ is d-separated with $v_j$ given $Z$ in $D$ is $2^{n-j}$. In total,  $|\{Z \subseteq V\setminus \{v_1, v_j\}: T_D(v_1, v_j, Z) \neq T_D(v_1, v_j, Z)\}| = 2^{n-j}$.

Similarly, for $j \in [n]\setminus \{1,2\}$ and $Z \subseteq V\setminus \{v_2, v_j\}$, if $v_2$ is d-connected with $v_j$ given $Z$ in $D'$ then $v_2$ is d-connected with $v_j$ given $Z$ in $D$.  
Also, if $v_1 \notin Z$ then if $v_2$ is d-connected with $v_j$ given $Z$ in $D$ then $v_2$ is d-connected with $v_j$ given $Z$ in $D'$. 
And, if $v_1 \in Z$ then if $v_2$ is d-connected with $v_j$ given $Z$ in $D$ then $v_2$ is d-separated with $v_j$ given $Z$ in $D'$. Such a $Z$ for which $v_1$ is d-connected with $v_j$ in $D$ and d-separated with $v_j$ in $D'$ is of the form $\mathcal{C}_{2,j} \cup \{v_1\} \cup Y$, where $Y \subseteq \{v_{j+1}, v_{j+2} \ldots, v_n\}$. Therefore, the number of $Z$ for which $v_2$ is d-connected with $v_j$ given $Z$ in $D$ and $v_2$ is d-separated with $v_j$ given $Z$ in $D'$ is $2^{n-j}$. In total, $|\{Z \subseteq V\setminus \{v_2, v_j\}: T_D(v_2, v_j, Z) \neq T_D(v_2, v_j, Z)\}| = 2^{n-j}$.

This implies that if $v_2$ is a collider node in $D$ then $d(D, D') = \sum_{3\leq j \leq n}{2^{n-j} + 2^{n-j}} = \sum_{3\leq j \leq n}{2^{n-j+1}} = 2(2^{n-2}-1) = 2^{n-1} - 2$.

Now suppose $v_2$ is not a collider in $D$.  We follow the similar argument as above to show that even in this case also $d(D, D') = 2^{n-1} - 2$. For all $3\leq j \leq n$ and $Z \subseteq V\setminus \{v_1, v_j\}$, if $v_1$ is d-connected with $v_j$ given $Z$ in $D$ then $v_1$ is d-connected with $v_j$ given $Z$ in $D'$. If $v_1$ is d-connected with $v_j$ given $Z$ in $D'$ and $v_2 \notin Z$ then $v_1$ is d-connected with $v_j$ given $Z$ in $D$. And, if $v_1$ is d-connected with $v_j$ given $Z$ in $D'$ and $v_2 \in Z$ then $v_1$ is d-separated with $v_j$ given $Z$ in $D$. Such a $Z$ for which $v_1$ is d-connected with $v_j$ in $D'$ and d-separated with $v_j$ in $D$ is of the form $\mathcal{C}_{2,j} \cup \{1\} \cup Y$, where $Y \subseteq \{v_{j+1}, v_{j+2} \ldots, v_n\}$. Therefore, the number of $Z$ for which $v_2$ is d-connected with $v_j$ given $Z$ in $D'$ and $v_2$ is d-separated with $v_j$ given $Z$ in $D$ is $2^{n-j}$. In total,  $|\{Z \subseteq V\setminus \{v_1, v_j\}: T_D(v_1, v_j, Z) \neq T_D(v_1, v_j, Z)\}| = 2^{n-j}$. Similarly, $|\{Z \subseteq V\setminus \{v_2, v_j\}: T_D(v_2, v_j, Z) \neq T_D(v_2, v_j, Z)\}| = 2^{n-j}$. Therefore, even when $v_2$ is not a collider in $D$, $d(D, D') = 2^{n-1} - 2$.

We show that, whether or not $u_2$ is a collider in $D$, there exists a Bayesian network $D'$ with $d(D, D') = 2^{n-1}-2$.
It remains to prove that no Bayesian network with a path skeleton is strictly closer to $D$ than $D'$.

Suppose there exists a Bayesian network $D''$ defined over $V$ with a path skeleton such that $D''$ is not Markov equivalent to $D$ and 
$d(D, D'') < 2^{n-1}-2$. First consider the case that the skeletons of $D$ and $D''$ are the same. 
As discussed in the introduction, two Bayesian networks are Markov equivalent iff they have the same skeleton and the same set of v-structures. For a Bayesian network with a path skeleton, $(u, v, w)$ is a v-structure in the Bayesian network iff $v$ is a collider in it. This implies that two Bayesian networks with a path skeleton are Markov equivalent iff they have the same skeleton and the same set of colliders. Therefore, $D$ and $D''$ must differ in at least one collider. Thus there exists 
$c$ with $1<c<n$ such that $v_c$ is a collider in exactly one of $D$ and $D''$.

Without loss of generality, assume $v_c$ is a collider in $D$ but not in $D''$. 
 Let $\mathcal{C}_{i,j}''$ be the set of colliders in $D''$ between nodes $v_i$ and $v_j$. From \Cref{lem:d-separation-in-paths}, for $1\leq i < j \leq n$ and $Z \subseteq V\setminus \{v_i, v_j\}$, $v_i$ and $v_j$ are d-connected in $D$ given $Z$ iff $Z \cap S_{i,j} = C_{i,j}$, and $v_i$ and $v_j$ are d-connected in $D''$ given $Z$ iff $Z \cap S_{i,j} = C_{i,j}''$. 
Since $v_c$ is a collider in $D$ but not in $D''$, for all $1\leq a < c$ and for all $c < b \leq n$, $\mathcal{C}_{a,b} \neq \mathcal{C}_{a,b}''$. 
If $Z \cap S_{a,b} = C_{a,b}$ then $Z$ is of the form $C_{a,b} \cup Y$ where $Y \subseteq V\setminus \{v_a, v_{a+1}, \ldots, v_b\}$, and for such $Z$, $v_a$ is d-connected with $v_b$ given $Z$ in $D$ and $v_a$ is d-separated with $v_b$ given $Z$ in $D''$.   Therefore, there are $2^{n-b+a-1}$ different $Z$ for which $v_a$ is d-connected with $v_b$ given $Z$ in $D$ but $v_b$ is d-separated with $v_a$ given $Z$ in $D''$. Similarly, for $Z$ of the form $C_{i,j}'' \cup Y$ where $Y \subseteq V\setminus \{v_a, v_{a+1}, \ldots, v_b\}$, $v_a$ is d-connected with $v_b$ given $Z$ in $D''$ and $v_a$ is d-separated with $v_b$ given $Z$ in $D$. Again, there are $2^{n-b+a-1}$ different $Z$ for which $v_a$ is d-connected with $v_b$ given $Z$ in $D''$ but $v_a$ is d-separated with $v_b$ given $Z$ in $D$. Therefore, in total, there are $2^{n-b+a}$ different $Z$ for which $T_{D}(v_a, v_b, Z) \neq T_{D''}(v_a, v_b, Z)$. Therefore, 

\begin{align*}
d(D, D'') &= \sum_{1\leq a < b \leq n} \big|\{Z \subseteq V \setminus \{v_a, v_b\}: T_{D}(v_a, v_b, Z) \neq T_{D''}(v_a, v_b, Z)\}\big| \\
  &\geq \sum_{1\leq a < c < b \leq n} \big|\{Z \subseteq V \setminus \{v_a, v_b\}: T_{D}(v_a, v_b, Z) \neq T_{D''}(v_a, v_b, Z)\}\big| \\
  &=  \sum_{1\leq a < c < b \leq n} {2^{n-b+a}} \\
  &= \sum_{1\leq a < c}{2^a} \times \sum_{c< b \leq n}{2^{n-b}} \\
  &= (2^c-2) \times (2^{n-c} - 1) \\
  &= 2^n + 2 - (2^c + 2^{n-c+1})
\end{align*}
The expression $2^c + 2^{n-c+1}$ is convex in $c$ for $2\leq c \leq n-1$ and is maximized at $c = 2$ or $c =n-1$. Therefore, $d(D, D'') \geq 2^n + 2 - (4 + 2^{n-1}) = 2^{n-1} - 2$.
This implies that there does not exist any Bayesian network $D''$ with the same skeleton as $D$ such that $D''$ is closer to $D$ than $D'$.

Now suppose the skeletons of $D$ and $D''$ are different. 
Then there must exist an edge that appears in exactly one of the two skeletons; otherwise, the skeletons would be identical, which contradicts our assumption. 
Without loss of generality, assume that $\{v_i,v_j\}$ is an edge in the $\skel{D}$ but not in the $\skel{D''}$.

Since $v_i$ and $v_j$ are adjacent in the $\skel{D}$, they are d-connected in $D$ for every set $Z \subseteq V \setminus \{v_i,v_j\}$. 

Because $\{v_i,v_j\}$ is not an edge of the $\skel{D''}$ and the skeleton of $D''$ is a path, there is a unique path between $v_i$ and $v_j$ in $D''$ that contains at least one internal node. 
Let $Y$ denote the set of nodes that lie strictly between $v_i$ and $v_j$ on this path. 
Then $Y \neq \emptyset$. 
Let $Y'$ be the set of nodes in $Y$ that are colliders in $D''$.
By \Cref{lem:d-separation-in-paths}, in $D''$ the nodes $v_i$ and $v_j$ are d-connected given $Z$ if and only if $Z \cap Y = Y'$. 
In other words, the set $Z$ must contain exactly the colliders in $Y$ and none of the other nodes in $Y$.

We now count how many such sets $Z$ exist. 
To construct a valid $Z$, we must include $Y'$, exclude $Y \setminus Y'$, and we may freely choose any subset of the remaining nodes $V \setminus (\{v_i,v_j\} \cup Y)$. 
The number of such sets is therefore $2^{\,n-2-|Y|}$.
Since the total number of subsets of $V \setminus \{v_i,v_j\}$ is $2^{\,n-2}$, the number of sets $Z$ for which $v_i$ and $v_j$ are d-connected in $D$ but not in $D''$ equals 
$2^{\,n-2} - 2^{\,n-2-|Y|}$. 
Because $|Y| \ge 1$, we have $2^{\,n-2} - 2^{\,n-2-|Y|} \ge 2^{\,n-3}$.

Thus, for this single missing edge, the two  Bayesian networks disagree on at least $2^{\,n-3}$ sets $Z$.
A symmetric argument applies to any edge that appears in the $\skel{D''}$ but not in the $\skel{D}$. 
If $t$ denotes the number of edges in the $\skel{D}$ that are absent from the $\skel{D''}$, then exactly $t$ edges appear in the $\skel{D''}$ but not in the $\skel{D}$. 
Each such edge contributes at least $2^{\,n-3}$ disagreements. 
Therefore $d(D, D'') \ge 2t \cdot 2^{\,n-3}$.

Since we assumed $d(D, D'') < 2^{\,n-1} - 2$, we must have $t = 1$. 
The skeleton of $D$ is the path shown in Figure~\ref{fig:path}. 
Hence there exists a unique index $1 \le c < n$ such that $\{v_c,v_{c+1}\}$ is an edge of the $\skel{D}$ but not of the $\skel{D''}$. 
For every $i \in [n-1]\setminus \{c\}$, the nodes $v_i$ and $v_{i+1}$ are adjacent in the skeletons of both
$D$ and the $D''$.

If we remove the edge $\{v_c,v_{c+1}\}$ from the $\skel{D}$, the path breaks into two smaller paths: one from $v_1$ to $v_c$, and one from $v_{c+1}$ to $v_n$. 
Because all other consecutive edges are shared by both skeletons, these two paths must also appear as induced subpaths in the $\skel{D''}$.

\input{fig-skeleton-of-M}

Since the skeleton of $D''$ is itself a path on all $n$ nodes, it must reconnect these two subpaths using exactly one new edge between their endpoints. 
Therefore, one of the following three cases must occur (see Figure~\ref{fig:skeleton-of-M''}):
(i) $\{v_c,v_n\}$ is an edge of the $\skel{D''}$ (i.e., the skeleton of $D''$ looks like Possibility~1 of Figure~\ref{fig:skeleton-of-M''});  
(ii) $\{v_1,v_{c+1}\}$ is an edge of the $\skel{D''}$ (i.e., the skeleton of $D''$ looks like Possibility~2 of Figure~\ref{fig:skeleton-of-M''}); or  
(iii) $\{v_1,v_n\}$ is an edge of the $\skel{D''}$ (i.e., the skeleton of $D''$ looks like Possibility~3 of Figure~\ref{fig:skeleton-of-M''}).

If $c=1$, then case (ii) implies the skeletons of $D$ and $D''$ are the same, a contradiction. 
Similarly, if $c=n-1$, then case (i) cannot occur.

Fix $1 \le a \le c$ and $c+1 \le b \le n$. 
We count the number of sets $Z \subseteq V \setminus \{v_a,v_b\}$ such that 
$v_a$ and $v_b$ are d-connected given $Z$ in exactly one of $D$ and $D''$. Let $Y_1 = \{v_i: 1\leq i < a\}$, $Y_2 = \{v_i: a < i \leq c\}$, $Y_3 = \{v_i: c+1 \leq i < b\}$, and $Y_4 = \{v_i: b < i \leq n\}$, as shown in Figure~\ref{fig:skeleton-of-M''}. This implies that $|Y_1| = a-1$, $|Y_2| = c-a$, $|Y_3| = b-c-1$, $|Y_4| = n-b$, and $|Y_1| + |Y_2| + |Y_3| + |Y_4| = n-2$.

Suppose $D''$ satisfies case (i), so $1 \le c < n-1$. For $Z \subseteq V\setminus \{v_a, v_b\}$, by \Cref{lem:d-separation-in-paths}, $v_a$ is d-connected with $v_b$ given $Z$ in $D$ iff $Z \cap (Y_2 \cup Y_3) = \mathcal{C}(D) \cap (Y_2 \cup Y_3)$. Following the previous approach, the number of such $Z$ is $2^{|Y_1| + |Y_4|}$. Similarly, in $D''$, $v_a$ is d-connected with $v_j$ given $Z$ iff $Z \cap (Y_2 \cup Y_4) = \mathcal{C}(D'') \cap (Y_2 \cup Y_4)$ (analogously to $\mathcal{C}(D)$, $\mathcal{C}(D'')$ denotes the set of colliders in $D''$). The number of such $Z$ is $2^{|Y_1| + |Y_3|}$. 

If $\mathcal{C}(D) \cap Y_2 \neq \mathcal{C}(D'') \cap Y_2$, then for all $Z \subseteq V\setminus \{v_a, v_b\}$ such that either $Z \cap (Y_2 \cup Y_3) = \mathcal{C}(D) \cap (Y_2 \cup Y_3)$ or $Z \cap (Y_2 \cup Y_4) = \mathcal{C}(D'') \cap (Y_2 \cup Y_4)$, $v_a$ is d-connected with $v_b$ given $Z$ in exactly one of $D$ and $D''$. The number of such $Z$ is $2^{|Y_1| + |Y_4|} + 2^{|Y_1| + |Y_3|} = 2^{|Y_1|}(2^{|Y_4|} + 2^{|Y_3|}) = 2^{a-1}(2^{n-b} + 2^{b-c-1})$.

Suppose $\mathcal{C}(D) \cap Y_2 = \mathcal{C}(D'') \cap Y_2$. For $Z \subseteq V\setminus \{v_a, v_b\}$,  $Z \cap (Y_2 \cup Y_3) = \mathcal{C}(D) \cap (Y_2 \cup Y_3)$ and $Z \cap (Y_2 \cup Y_4) = \mathcal{C}(D'') \cap (Y_2 \cup Y_4)$ iff $Z \cap (Y_2 \cup Y_3 \cup Y_4) = (\mathcal{C}(D) \cap (Y_2 \cup Y_3)) \cup \mathcal{C}(D'') \cap (Y_2 \cup Y_4)$. The number of such $Z$ is $2^{|Y_1|}$. Therefore, the number of $Z$ for which $v_a$ is d-connected with $v_b$ given $Z$ in $D$ but d-separated in $D''$ is $2^{|Y_1| + |Y_4|} - 2^{|Y_1|} = 2^{|Y_1|}(2^{|Y_4|} - 1) = 2^{a-1}(2^{n-b} - 1)$. Similarly, the number of $Z$ for which $v_a$ is d-connected with $v_b$ given $Z$ in $D''$ but d-separated in $D$ is $2^{|Y_1| + |Y_3|} - 2^{|Y_1|} = 2^{|Y_1|}(2^{|Y_3|} - 1) = 2^{a-1}(2^{b-c-1} - 1)$. This implies there are $2^{a-1}(2^{n-b} + 2^{b-c-1} - 2)$ many $Z$ for which $v_a$ is d-connected with $v_b$ given $Z$ in exactly one of $D$ and $D''$. 

This implies that there exists at least $2^{a-1}(2^{n-b} + 2^{b-c-1} - 2)$ many $Z$ such that $v_a$ is d-connected with $v_b$ given $Z$ in exactly one of $D$ and $D''$. This implies $|\{Z \subseteq V\setminus \{v_a, v_b\}: T_D(v_a, v_b, Z) \neq T_{D''}(v_a, v_b, Z)\}| \geq 2^{a-1}(2^{n-b} + 2^{b-c-1} - 2)$. Therefore,          
\begin{align*}
    d(D, D'') &= \sum_{1\leq a < b \leq n} \big|\{Z \subseteq V \setminus \{v_a, v_b\}: T_{D}(v_a, v_b, Z) \neq T_{D''}(v_a, v_b, Z)\}\big| \\
  &\geq \sum_{a = 1}^{c}{\sum_{b = c+1}^{n}}  \big|\{Z \subseteq V \setminus \{v_a, v_b\}: T_{D}(v_a, v_b, Z) \neq T_{D''}(v_a, v_b, Z)\}\big| \\
  &= \sum_{a = 1}^{c}{\sum_{b = c+1}^{n}{2^{a-1}(2^{n-b} + 2^{b-c-1} - 2)}}\\
              &= \sum_{a = 1}^{c}{2^{a-1}} \times \sum_{b = c+1}^{n}{2^{n-b} + 2^{b-c-1} - 2}\\
              &= (2^c - 1) \times ((2^{n-c} - 1) + (2^{n-c} -1) -2(n-c))\\
              &= 2(2^c -1) \times (2^{n-c} - 1 - (n-c)).
\end{align*}

Since $2^{n-c} - (n-c) \ge 2^{n-c-1}$, we obtain $d(D, D'') \ge 2(2^c - 1)(2^{n-c-1} - 1) 
= 2(2^{n-1} + 1 - (2^c + 2^{n-c-1}))$.

The expression $2^c + 2^{n-c-1}$ is convex in $c$ for $1 \le c \le n-2$ 
and is maximized at $c=1$ or $c=n-2$. 
Therefore, $d(D, D'') \ge 2(2^{n-1} + 1 - (2^{n-2} + 2)) 
= 2^{n-1} - 2$,
which contradicts the assumption $d(D, D'') < 2^{n-1} - 2$. 
Hence, $D''$ cannot satisfy case (i).

Case (ii) is symmetric to case (i). 
The difference is that the roles of $Y_1$ and $Y_2$ are interchanged, 
and the roles of $Y_3$ and $Y_4$ are interchanged (see Possibilities~ 1 and 2 of Figure~\ref{fig:skeleton-of-M''}). 
In this case, the set of common nodes lying between $v_a$ and $v_b$ 
in both $D$ and $D''$ is $Y_3$, rather than $Y_2$ as in case (i).
After this relabeling, the counting argument is identical to that in case (i). 
Therefore, by the same reasoning, 
$D''$ cannot satisfy case (ii). For the sake of completeness, we are providing the proof.

Suppose $D''$ satisfies case (ii), so $1< c \leq n-1$. For $Z \subseteq V\setminus \{v_a, v_b\}$, by \Cref{lem:d-separation-in-paths}, $v_a$ is d-connected with $v_b$ given $Z$ in $D$ iff $Z \cap (Y_2 \cup Y_3) = \mathcal{C}(D) \cap (Y_2 \cup Y_3)$. The number of such $Z$ is $2^{|Y_1| + |Y_4|}$. Similarly, in $D''$, $v_a$ is d-connected with $v_j$ given $Z$ iff $Z \cap (Y_1 \cup Y_3) = \mathcal{C}(D'') \cap (Y_1 \cup Y_3)$. The number of such $Z$ is $2^{|Y_2| + |Y_4|}$. 

If $\mathcal{C}(D) \cap Y_3 \neq \mathcal{C}(D'') \cap Y_3$, then for all $Z \subseteq V\setminus \{v_a, v_b\}$ such that either $Z \cap (Y_2 \cup Y_3) = \mathcal{C}(D) \cap (Y_2 \cup Y_3)$ or $Z \cap (Y_1 \cup Y_3) = \mathcal{C}(D'') \cap (Y_1 \cup Y_3)$, $v_a$ is d-connected with $v_b$ given $Z$ in exactly one of $D$ and $D''$. The number of such $Z$ is $2^{|Y_1| + |Y_4|} + 2^{|Y_2| + |Y_4|} = 2^{|Y_4|}(2^{|Y_1|} + 2^{|Y_2|}) = 2^{n-b}(2^{a-1} + 2^{c-a})$.

Suppose $\mathcal{C}(D) \cap Y_3 = \mathcal{C}(D'') \cap Y_3$. For $Z \subseteq V\setminus \{v_a, v_b\}$,  $Z \cap (Y_2 \cup Y_3) = \mathcal{C}(D) \cap (Y_2 \cup Y_3)$ and $Z \cap (Y_1 \cup Y_3) = \mathcal{C}(D'') \cap (Y_1 \cup Y_3)$ iff $Z \cap (Y_1 \cup Y_2 \cup Y_3) = (\mathcal{C}(D) \cap (Y_2 \cup Y_3)) \cup \mathcal{C}(D'') \cap (Y_1 \cup Y_3)$. The number of such $Z$ is $2^{|Y_4|}$. Therefore, the number of $Z$ for which $v_a$ is d-connected with $v_b$ given $Z$ in $D$ but d-separated in $D''$ is $2^{|Y_1| + |Y_4|} - 2^{|Y_4|} = 2^{|Y_4|}(2^{|Y_1|} - 1) = 2^{n-b}(2^{a-1} - 1)$. Similarly, the number of $Z$ for which $v_a$ is d-connected with $v_b$ given $Z$ in $D''$ but d-separated in $D$ is $2^{|Y_2| + |Y_4|} - 2^{|Y_4|} = 2^{|Y_4|}(2^{|Y_2|} - 1) = 2^{n-b}(2^{c-a} - 1)$. This implies there are $2^{n-b}(2^{a-1} + 2^{c-a} - 2)$ many $Z$ for which $v_a$ is d-connected with $v_b$ given $Z$ in exactly one of $D$ and $D''$. 

This implies that there exists at least $2^{n-b}(2^{a-1} + 2^{c-a} - 2)$ many $Z$ such that $v_a$ is d-connected with $v_b$ given $Z$ in exactly one of $D$ and $D''$. This implies $|\{Z \subseteq V\setminus \{v_a, v_b\}: T_D(v_a, v_b, Z) \neq T_{D''}(v_a, v_b, Z)\}| \geq 2^{n-b}(2^{a-1} + 2^{c-a} - 2)$. Therefore,  
\begin{align*}
    d(D, D'') &= \sum_{1\leq a < b \leq n} \big|\{Z \subseteq V \setminus \{v_a, v_b\}: T_{D}(v_a, v_b, Z) \neq T_{D''}(v_a, v_b, Z)\}\big| \\
  &\geq \sum_{a = 1}^{c}{\sum_{b = c+1}^{n}}  \big|\{Z \subseteq V \setminus \{v_a, v_b\}: T_{D}(v_a, v_b, Z) \neq T_{D''}(v_a, v_b, Z)\}\big| \\
  &= \sum_{a = 1}^{c}{\sum_{b = c+1}^{n}{2^{n-b}(2^{a-1} + 2^{c-a} - 2)}}\\
              &= \sum_{a = 1}^{c}{(2^{a-1} + 2^{c-a} - 2)} \times \sum_{b = c+1}^{n}{2^{n-b}}\\
              &= ((2^c - 1) + (2^c-1) - 2c) \times (2^{n-c} - 1)\\
              &= 2(2^c -1 - c) \times (2^{n-c} - 1).
\end{align*}
Since $2^{c} - c \ge 2^{c-1}$ for $2\leq c\leq n-1$, we obtain $d(D, D'') \ge 2(2^{c -1}- 1)(2^{n-c} - 1) 
= 2(2^{n-1} + 1 - (2^{c-1} + 2^{n-c}))$.

The expression $2^{c-1} + 2^{n-c}$ is convex in $c$ for $2 \le c \le n-1$ 
and is maximized at $c=2$ or $c=n-1$. 
Therefore, $d(D, D'') \ge 2(2^{n-1} + 1 - (2^{n-2} + 2)) 
= 2^{n-1} - 2$,
which contradicts the assumption $d(D, D'') < 2^{n-1} - 2$. 
Hence, $D''$ cannot satisfy case (ii).

Now, suppose $D''$ satisfies case (iii). 
For $Z \subseteq V\setminus \{v_a, v_b\}$, by \Cref{lem:d-separation-in-paths}, $v_a$ is d-connected with $v_b$ given $Z$ in $D$ iff $Z \cap (Y_2 \cup Y_3) = \mathcal{C}(D) \cap (Y_2 \cup Y_3)$. The number of such $Z$ is $2^{|Y_1| + |Y_4|}$. Similarly, in $D''$, $v_a$ is d-connected with $v_j$ given $Z$ iff $Z \cap (Y_1 \cup Y_4) = \mathcal{C}(D'') \cap (Y_1 \cup Y_4)$. The number of such $Z$ is $2^{|Y_2| + |Y_3|}$. 

 For $Z \subseteq V\setminus \{v_a, v_b\}$,  $Z \cap (Y_2 \cup Y_3) = \mathcal{C}(D) \cap (Y_2 \cup Y_3)$ and $Z \cap (Y_1 \cup Y_4) = \mathcal{C}(D'') \cap (Y_1 \cup Y_4)$ iff $Z \cap (Y_1 \cup Y_2 \cup Y_3 \cup Y_4) = (\mathcal{C}(D) \cap (Y_2 \cup Y_3)) \cup \mathcal{C}(D'') \cap (Y_1 \cup Y_4)$. $Z = (\mathcal{C}(D) \cap (Y_2 \cup Y_3)) \cup \mathcal{C}(D'') \cap (Y_1 \cup Y_4)$ is the only such $Z$. Therefore, the number of $Z$ for which $v_a$ is d-connected with $v_b$ given $Z$ in $D$ but d-separated in $D''$ is $2^{|Y_1| + |Y_4|} - 1 = 2^{n+a-b-1} - 1$. Similarly, the number of $Z$ for which $v_a$ is d-connected with $v_b$ given $Z$ in $D''$ but d-separated in $D$ is $2^{|Y_2| + |Y_3|} - 1 =  2^{b-a-1} - 1$. This implies there are $2^{n+a-b-1} + 2^{b-a-1} - 2$ many $Z$ for which $v_a$ is d-connected with $v_b$ given $Z$ in exactly one of $D$ and $D''$. 
 This implies $|\{Z \subseteq V\setminus \{v_a, v_b\}: T_D(v_a, v_b, Z) \neq T_{D''}(v_a, v_b, Z)\}| = 2^{n+a-b-1} + 2^{b-a-1} - 2$. Therefore,  
\begin{align*}
    d(D, D'') &= \sum_{1\leq a < b \leq n} \big|\{Z \subseteq V \setminus \{v_a, v_b\}: T_{D}(v_a, v_b, Z) \neq T_{D''}(v_a, v_b, Z)\}\big| \\
  &\geq \sum_{a = 1}^{c}{\sum_{b = c+1}^{n}}  \big|\{Z \subseteq V \setminus \{v_a, v_b\}: T_{D}(v_a, v_b, Z) \neq T_{D''}(v_a, v_b, Z)\}\big| \\
  &= \sum_{a = 1}^{c}{\sum_{b = c+1}^{n}{2^{n+a-b-1} + 2^{b-a-1} - 2}}\\
              &= \sum_{a = 1}^{c}{2^{a-1}(2^{n-c} - 1) + 2^{c-a}(2^{n-c} - 1) - 2(n-c))}\\
              &= (2^{n-c} - 1)(2^c - 1) + (2^{n-c} - 1) (2^c -1) - 2(n-c)c\\
              &= 2((2^{n-c} - 1)(2^c - 1) -  (n-c)c)\\
              &= 2(2^n + 1 - (2^{n-c} + 2^c + (n-c) c)).
\end{align*}

The expression $2^{n-c} + 2^c + (n-c) c$ is convex in $c$ for $1 \le c \le n-1$ 
and is maximized at $c=1$ or $c=n-1$. 
Therefore, $d(D, D'') \ge 2(2^n + 1 - (2^{n-1} + 2 + (n-1))) =  2(2^{n-1} - n) = (2^{n-1}-2) + (2^{n-1} - 2(n-1))$.

For $n \geq 3$, $2^{n-1} - 2(n-1) \geq 0$. Therefore, $d(D, D') \geq 2^{n-1}-2$,
which contradicts the assumption $d(D, D'') < 2^{n-1} - 2$. 
Hence, $D''$ cannot satisfy case (iii).

This shows that there is no Bayesian network $D''$ with a path skeleton such that 
$d(D, D'') < d(D,D')$. 
Hence, $D'$ is the closest Bayesian network to $D$ among all Bayesian networks with a path skeleton. 
This completes the proof of Theorem~\ref{thm:closest-MEC-for-path-graphs-complete}.
\end{proof}

%% file: fig-path.tex
\begin{figure}[t]
\centering
\begin{tikzpicture}[
    scale=0.9,
    transform shape,
    every node/.style={circle, draw, minimum size=7mm, inner sep=0pt}
]

\node[draw=red] (u1) at (0,0) {$v_1$};
\node[draw=none, left=0.1cm of u1] (u0) {skeleton of $D$:};

\node[draw=red,  right=0.5cm of u1] (u2) {$v_2$};
\node[right=0.5cm of u2] (u3) {$v_3$};
\node[  right=0.5cm of u3] (u4) {$v_4$};
\node[right=0.5cm of u4] (u5) {$\ldots$};
\node[right=0.5cm of u5] (u6) {$v_n$};


\draw[-, thick] (u1) -- (u2);
\draw[-, thick] (u2) -- (u3);
\draw[-, thick] (u3) -- (u4);
\draw[-, thick] (u4) -- (u5);
\draw[thick] (u5) -- (u6);

\node[draw=red, below=0.5cm of u1] (v2) {$v_2$};
\node[draw=none, left=0.1cm of v2] (v0) {skeleton of $D'$:};
\node[draw=red, right=0.5cm of v2] (v1) {$v_1$};
\node[right=0.5cm of v1] (v3) {$v_3$};
\node[  right=0.5cm of v3] (v4) {$v_4$};
\node[right=0.5cm of v4] (v5) {$\ldots$};
\node[right=0.5cm of v5] (v6) {$v_n$};

\draw[thick] (v2) -- (v1);
\draw[thick] (v1) -- (v3);
\draw[-, thick] (v3) -- (v4);
\draw[-, thick] (v4) -- (v5);
\draw[thick] (v5) -- (v6);

\end{tikzpicture}
\caption{Skeletons of $D$ and $D'$.}
\label{fig:path}
\end{figure}
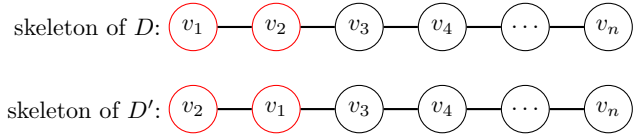

%% file: fig-path-example.tex
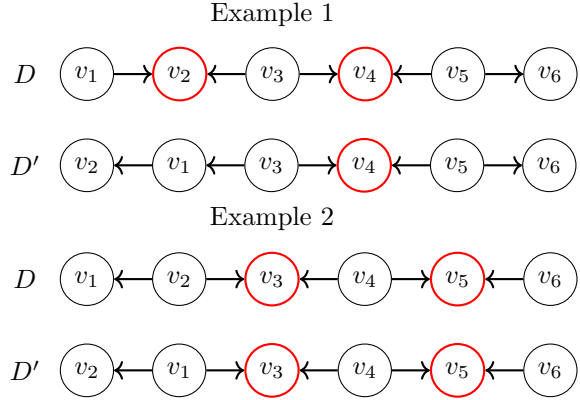
\begin{figure}[t]
\centering
\begin{tikzpicture}[
    scale=1,
    transform shape,
    every node/.style={circle, draw, minimum size=7mm, inner sep=0pt}
]


\node (u1) at (0,0) {$v_1$};
\node[draw=none, left=0.1cm of u1] (u0) {$D$};
\node[draw=none, below=0.5cm of u0] {$D'$};

\node[draw=red, thick, right=0.5cm of u1] (u2) {$v_2$};
\node[right=0.5cm of u2] (u3) {$v_3$};
\node[draw=red, thick, right=0.5cm of u3] (u4) {$v_4$};
\node[right=0.5cm of u4] (u5) {$v_5$};
\node[right=0.5cm of u5] (u6) {$v_6$};

\node[draw=none, above= -0.4cm of u3] {Example 1};

\draw[->, thick] (u1) -- (u2);
\draw[<-, thick] (u2) -- (u3);
\draw[->, thick] (u3) -- (u4);
\draw[<-, thick] (u4) -- (u5);
\draw[->, thick] (u5) -- (u6);

\node[below=0.5cm of u1] (v2) {$v_2$};
\node[ right=0.5cm of v2] (v1) {$v_1$};
\node[right=0.5cm of v1] (v3) {$v_3$};
\node[draw=red, thick, right=0.5cm of v3] (v4) {$v_4$};
\node[right=0.5cm of v4] (v5) {$v_5$};
\node[right=0.5cm of v5] (v6) {$v_6$};

\draw[<-, thick] (v2) -- (v1);
\draw[<-, thick] (v1) -- (v3);
\draw[->, thick] (v3) -- (v4);
\draw[<-, thick] (v4) -- (v5);
\draw[->, thick] (v5) -- (v6);


\begin{scope}[xshift=6cm]

\node[below=0.8cm of v2] (a1) {$v_1$};
\node[draw=none, left=0.1cm of a1] (a0) {$D$};
\node[draw=none, below=0.5cm of a0] {$D'$};

\node[right=0.5cm of a1] (a2) {$v_2$};
\node[draw=red, thick, right=0.5cm of a2] (a3) {$v_3$};
\node[right=0.5cm of a3] (a4) {$v_4$};
\node[draw=red, thick, right=0.5cm of a4] (a5) {$v_5$};
\node[right=0.5cm of a5] (a6) {$v_6$};

\node[draw=none, above=-0.4cm of a3] {Example 2};

\draw[<-, thick] (a1) -- (a2);
\draw[->, thick] (a2) -- (a3);
\draw[<-, thick] (a3) -- (a4);
\draw[->, thick] (a4) -- (a5);
\draw[<-, thick] (a5) -- (a6);

\node[below=0.5 cm of a1] (b1) {$v_2$};
\node[right=0.5cm of b1] (b2) {$v_1$};
\node[draw=red, thick, right=0.5cm of b2] (b3) {$v_3$};
\node[right=0.5cm of b3] (b4) {$v_4$};
\node[draw=red, thick, right=0.5cm of b4] (b5) {$v_5$};
\node[right=0.5cm of b5] (b6) {$v_6$};

\draw[<-, thick] (b1) -- (b2);
\draw[->, thick] (b2) -- (b3);
\draw[<-, thick] (b3) -- (b4);
\draw[->, thick] (b4) -- (b5);
\draw[<-, thick] (b5) -- (b6);

\end{scope}

\end{tikzpicture}
\caption{Examples of $D$ and $D'$.}
\label{fig:path-example}
\end{figure}

%% file: fig-skeleton-of-M.tex
\begin{figure}[t]
\centering
\begin{tikzpicture}[
    scale=0.6,
    transform shape,
    every node/.style={circle, draw, minimum size=7mm, inner sep=0pt}
]

\node[draw=none] (x0) at (0,0) {$\skel{D}$:};
\node[draw=none, below=0cm of x0] (a0) {Possibility 1:};
\node[draw=none, below=0cm of a0] (b0) {Possibility 2:};
\node[draw=none, below=0cm of b0] (c0) {Possibility 3:};

\node[right=0.5cm of x0] (x1) {$v_1$};
\node[right=0.5cm of x1] (x2) {$\ldots$};
\node[right=0.5cm of x2] (x3) {$v_a$};
\node[right=0.5cm of x3] (x4) {$\ldots$};
\node[right=0.5cm of x4] (x5) {$v_c$};
\node[right=0.5cm of x5] (x6) {$v_{c+1}$};
\node[right=0.5cm of x6] (x7) {$\ldots$};
\node[right=0.5cm of x7] (x8) {$v_b$};
\node[right=0.5cm of x8] (x9) {$\ldots$};
\node[right=0.5cm of x9] (x10) {$v_n$};
\draw[-, thick] (x1) -- (x2);
\draw[-, thick] (x2) -- (x3);
\draw[-, thick] (x3) -- (x4);
\draw[-, thick] (x4) -- (x5);
\draw[-, thick] (x5) -- (x6);
\draw[-, thick] (x6) -- (x7);
\draw[-, thick] (x7) -- (x8);
\draw[-, thick] (x8) -- (x9);
\draw[-, thick] (x9) -- (x10);
\node[
    rectangle,
    draw=red,
    dash pattern=on 2pt off 2pt,
    line width=0.8pt,
    rounded corners,
    fit=(x1)(x2),
    inner sep=2pt
] (R1) {};

\node[
    rectangle,
    draw=blue,
    dash pattern=on 2pt off 2pt,
    line width=0.8pt,
    rounded corners,
    fit=(x4)(x5),
    inner sep=2pt
] (R2) {};

\node[
    rectangle,
    draw=green,
    dash pattern=on 2pt off 2pt,
    line width=0.8pt,
    rounded corners,
    fit=(x6)(x7),
    inner sep=2pt
] (R3) {};

\node[
    rectangle,
    draw=orange,
    dash pattern=on 2pt off 2pt,
    line width=0.8pt,
    rounded corners,
    fit=(x9)(x10),
    inner sep=2pt
] (R4) {};

\node[draw=none, above=-4pt of R1] {$Y_1$};
\node[draw=none, above=-4pt of R2] {$Y_2$};
\node[draw=none, above=-4pt of R3] {$Y_3$};
\node[draw=none, above=-4pt of R4] {$Y_4$};

\node[right=0.5cm of a0] (a1) {$v_1$};
\node[right=0.5cm of a1] (a2) {$\ldots$};
\node[right=0.5cm of a2] (a3) {$v_a$};
\node[right=0.5cm of a3] (a4) {$\ldots$};
\node[right=0.5cm of a4] (a5) {$v_c$};
\node[right=0.5cm of a5] (a6) {$v_{n}$};
\node[right=0.5cm of a6] (a7) {$\ldots$};
\node[right=0.5cm of a7] (a8) {$v_b$};
\node[right=0.5cm of a8] (a9) {$\ldots$};
\node[right=0.5cm of a9] (a10) {$v_{c+1}$};
\draw[-, thick] (a1) -- (a2);
\draw[-, thick] (a2) -- (a3);
\draw[-, thick] (a3) -- (a4);
\draw[-, thick] (a4) -- (a5);
\draw[-, thick] (a5) -- (a6);
\draw[-, thick] (a6) -- (a7);
\draw[-, thick] (a7) -- (a8);
\draw[-, thick] (a8) -- (a9);
\draw[-, thick] (a9) -- (a10);
\node[
    rectangle,
    draw=red,
    dash pattern=on 2pt off 2pt,
    line width=0.8pt,
    rounded corners,
    fit=(a1)(a2),
    inner sep=2pt
] (R11) {};

\node[
    rectangle,
    draw=blue,
    dash pattern=on 2pt off 2pt,
    line width=0.8pt,
    rounded corners,
    fit=(a4)(a5),
    inner sep=2pt
] (R21) {};

\node[
    rectangle,
    draw=orange,
    dash pattern=on 2pt off 2pt,
    line width=0.8pt,
    rounded corners,
    fit=(a6)(a7),
    inner sep=2pt
] (R31) {};

\node[
    rectangle,
    draw=green,
    dash pattern=on 2pt off 2pt,
    line width=0.8pt,
    rounded corners,
    fit=(a9)(a10),
    inner sep=2pt
] (R41) {};

\node[draw=none, above=-4pt of R11] {$Y_1$};
\node[draw=none, above=-4pt of R21] {$Y_2$};
\node[draw=none, above=-4pt of R31] {$Y_4$};
\node[draw=none, above=-4pt of R41] {$Y_3$};
\node[right=0.5cm of b0] (b1) {$v_c$};
\node[right=0.5cm of b1] (b2) {$\ldots$};
\node[right=0.5cm of b2] (b3) {$v_a$};
\node[right=0.5cm of b3] (b4) {$\ldots$};
\node[right=0.5cm of b4] (b5) {$v_1$};
\node[right=0.5cm of b5] (b6) {$v_{c+1}$};
\node[right=0.5cm of b6] (b7) {$\ldots$};
\node[right=0.5cm of b7] (b8) {$v_b$};
\node[right=0.5cm of b8] (b9) {$\ldots$};
\node[right=0.5cm of b9] (b10) {$v_{n}$};
\draw[-, thick] (b1) -- (b2);
\draw[-, thick] (b2) -- (b3);
\draw[-, thick] (b3) -- (b4);
\draw[-, thick] (b4) -- (b5);
\draw[-, thick] (b5) -- (b6);
\draw[-, thick] (b6) -- (b7);
\draw[-, thick] (b7) -- (b8);
\draw[-, thick] (b8) -- (b9);
\draw[-, thick] (b9) -- (b10);
\node[
    rectangle,
    draw=blue,
    dash pattern=on 2pt off 2pt,
    line width=0.8pt,
    rounded corners,
    fit=(b1)(b2),
    inner sep=2pt
] (R12) {};

\node[
    rectangle,
    draw=red,
    dash pattern=on 2pt off 2pt,
    line width=0.8pt,
    rounded corners,
    fit=(b4)(b5),
    inner sep=2pt
] (R22) {};

\node[
    rectangle,
    draw=green,
    dash pattern=on 2pt off 2pt,
    line width=0.8pt,
    rounded corners,
    fit=(b6)(b7),
    inner sep=2pt
] (R32) {};

\node[
    rectangle,
    draw=orange,
    dash pattern=on 2pt off 2pt,
    line width=0.8pt,
    rounded corners,
    fit=(b9)(b10),
    inner sep=2pt
] (R42) {};

\node[draw=none, above=-4pt of R12] {$Y_2$};
\node[draw=none, above=-4pt of R22] {$Y_1$};
\node[draw=none, above=-4pt of R32] {$Y_3$};
\node[draw=none, above=-4pt of R42] {$Y_4$};

\node[right=0.5cm of c0] (c1) {$v_c$};
\node[right=0.5cm of c1] (c2) {$\ldots$};
\node[right=0.5cm of c2] (c3) {$v_a$};
\node[right=0.5cm of c3] (c4) {$\ldots$};
\node[right=0.5cm of c4] (c5) {$v_1$};
\node[right=0.5cm of c5] (c6) {$v_{n}$};
\node[right=0.5cm of c6] (c7) {$\ldots$};
\node[right=0.5cm of c7] (c8) {$v_b$};
\node[right=0.5cm of c8] (c9) {$\ldots$};
\node[right=0.5cm of c9] (c10) {$v_{c+1}$};
\draw[-, thick] (c1) -- (c2);
\draw[-, thick] (c2) -- (c3);
\draw[-, thick] (c3) -- (c4);
\draw[-, thick] (c4) -- (c5);
\draw[-, thick] (c5) -- (c6);
\draw[-, thick] (c6) -- (c7);
\draw[-, thick] (c7) -- (c8);
\draw[-, thick] (c8) -- (c9);
\draw[-, thick] (c9) -- (c10);
\node[
    rectangle,
    draw=blue,
    dash pattern=on 2pt off 2pt,
    line width=0.8pt,
    rounded corners,
    fit=(c1)(c2),
    inner sep=2pt
] (R13) {};

\node[
    rectangle,
    draw=red,
    dash pattern=on 2pt off 2pt,
    line width=0.8pt,
    rounded corners,
    fit=(c4)(c5),
    inner sep=2pt
] (R23) {};

\node[
    rectangle,
    draw=orange,
    dash pattern=on 2pt off 2pt,
    line width=0.8pt,
    rounded corners,
    fit=(c6)(c7),
    inner sep=2pt
] (R33) {};

\node[
    rectangle,
    draw=green,
    dash pattern=on 2pt off 2pt,
    line width=0.8pt,
    rounded corners,
    fit=(c9)(c10),
    inner sep=2pt
] (R43) {};

\node[draw=none, above=-4pt of R13] {$Y_2$};
\node[draw=none, above=-4pt of R23] {$Y_1$};
\node[draw=none, above=-4pt of R33] {$Y_4$};
\node[draw=none, above=-4pt of R43] {$Y_3$};

\end{tikzpicture}
\caption{Possible skeletons of $D''$.}
\label{fig:skeleton-of-M''}
\end{figure}
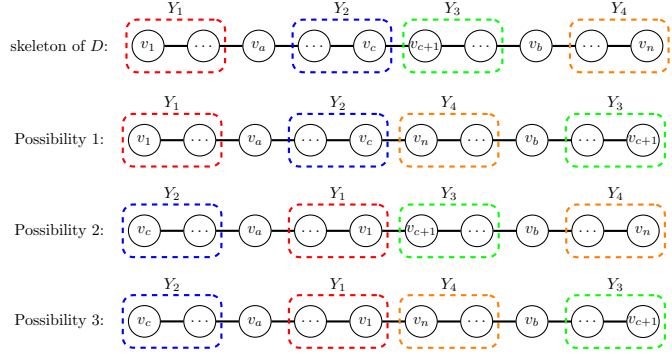